\documentclass[preprints,article,accept,pdftex,moreauthors]{Definitions/mdpi} 

\usepackage{amsmath}
\usepackage{amssymb}
\usepackage{mathtools}
\usepackage{xcolor}
\usepackage{microtype}
\usepackage{booktabs}
\usepackage{multirow}
\usepackage{siunitx}
\usepackage{makecell}
\usepackage{xr}
\newcommand{\mcref}[1]{\cref{#1}}
\newcommand{\scref}[1]{\cref{#1}}
\usepackage{listings}
\newcommand{\R}{\mathbb{R}}

\DeclareMathOperator{\Tr}{Tr}
\DeclareMathOperator{\diag}{diag}
\newcommand{\W}{\mathcal{W}}
\newcommand{\kernel}{\mathcal{K}}

\newcommand{\xp}[1]{{\widetilde #1}}

\newcommand{\supp}[1]{Supplementary Note~\ref{#1}} 

\newcommand{\etal}{\textit{et al.}}

\newcommand{\emd}{classical OT}
\newcommand{\wax}{$\mathcal{W}$aX}

\firstpage{1} 
\pubvolume{1}
\issuenum{1}
\articlenumber{0}
\pubyear{2026}
\copyrightyear{2026}
\datereceived{ } 
\daterevised{ } 
\dateaccepted{ } 
\datepublished{ } 

\Title{Reliable Modeling of Distribution Shifts via Displacement-Reshaped Optimal Transport}

\newcommand{\mnamelong}{Displacement-Reshaped Optimal Transport}
\newcommand{\mname}{ReshapeOT}

\Author{Philip Naumann $^{1,2}$\orcidA{}, Jacob Kauffmann $^{1,2}$\orcidB{}, Klaus-Robert M\"uller $^{1,2,4,5}$*\orcidC{}, and Gr\'egoire Montavon $^{1,3}$*\orcidD{}}

\AuthorNames{Philip Naumann, Jacob Kauffmann, Klaus-Robert M\"uller, and Gr\'egoire Montavon}

\address{%
$^{1}$ \quad BIFOLD\,--\,Berlin Institute for the Foundations of Learning and Data, 10587 Berlin, Germany\\
$^{2}$ \quad Machine Learning Group, Technische Universit\"at Berlin, 10587 Berlin, Germany\\
$^{3}$ \quad Institute for AI in Medicine, Charit\'e\,--\,Universit\"atsmedizin Berlin, 10117 Berlin, Germany\\
$^{4}$ \quad Department of Artificial Intelligence, Korea University, Anam-dong, Seongbuk-gu, Seoul 02841, Korea\\
$^{5}$ \quad Max Planck Institute for Informatics, Stuhlsatzenhausweg, 66123 Saarbr{\"u}cken, Germany}

\corres{Correspondence: klaus-robert.mueller@tu-berlin.de, gregoire.montavon@charite.de}

\abstract{%
Optimal transport (OT) is a central framework for modeling distribution shifts. Because OT compares distributions directly in input space, a well-designed ground metric between observations is essential to ensure that the optimizer does not violate the true geometry of change.
We propose \mnamelong{} (\mname{}), a method that reshapes the ground metric by integrating observed sample displacements as an additional source of knowledge. Technically, \mname{} replaces the Euclidean metric with a Mahalanobis distance estimated from displacement second moments. This effectively carves expressways through the input space, inviting transport solutions that better align with observed displacements. Our method is computationally lightweight, integrates seamlessly into any OT solver that operates on a cost matrix, and can be kernelized for further flexibility.
Experiments on synthetic and real-world data show that \mname{} achieves substantial gains in transport reliability. We further demonstrate our method's usefulness in two practical use cases.
}%


\begin{document}

\section{Introduction}
\label{sec:introduction}

\noindent Modeling distribution shifts is a critical research area (cf.~\cite{DBLP:journals/jmlr/SugiyamaKM07,quionero-candelaDatasetShiftMachine2009}) with applications ranging from characterizing real-world phenomena to developing robust machine learning models. Information geometry (IG)~\cite{Amari2016} provides a rigorous theoretical framework for these shifts by representing probability distributions as points on a Riemannian manifold, where shifts are quantified as geodesic distances. While IG traditionally characterizes shifts in parameter space (cf.~\cite{amari2000methods}), recent work has pivoted toward measuring shifts relative to the distribution's support~\cite{Otto2001}. This perspective draws deep connections to optimal transport~(OT)~\cite{Otto2001,villaniOptimalTransportOld2008, Peyre2019}, facilitating diverse applications such as tracking developmental processes in single-cell data~\cite{schiebinger2019optimal, kleinMappingCellsTime2025} and building robust ML models~\cite{montavon2016wasserstein, courtyOptimalTransportDomain2017, Andeol2023}.

\smallskip

Due to its focus on the distribution's support, OT requires special attention to the function used to measure distance in input space. For instance, using the Euclidean distance to measure distance between sources and targets tacitly assumes that transporting points over a straight line is a feasible and cost-efficient solution. This ignores practical structure: real displacements may follow corridors (e.g., migration routes, vascular pathways), avoid obstacles, or be governed by costs that differ substantially from raw Euclidean distance.

\smallskip

Much of the recent literature on optimal transport has focused on designing improved measures of transport cost to serve as optimization objectives. This includes objectives that integrate the manifold structure into the distance calculation (i.e.~geodesic distances)~\cite{HuguetTZTWK23}, that prevent transport plans from having discontinuities in their source-target mapping (see e.g.~\cite{courtyOptimalTransportDomain2017}), or that force an alignment to the dominant directions of variation in the data~\cite{patySubspaceRobustWasserstein2019}. While these first-principles formulations are highly effective at guiding the transport solution towards the specified behavior, they may lack flexibility in handling the multitude of real-world factors that make up the true transport costs.

\smallskip

Rather than handcrafting a suitable ground metric, we propose \textit{\mnamelong{}}, or, short, \textit{\mname{}}, which learns the metric directly from observed displacement statistics, steering OT toward empirically grounded solutions and away from naive Euclidean extrapolations.
Technically, our approach estimates second-order statistics of past displacements to derive a new, Mahalanobis-type ground metric that carves `expressways' into the original Euclidean ground metric. Our method easily integrates with existing OT solvers and incurs no significant increase in computational cost.
Because \mname{} only modifies the cost matrix and leaves the transport problem's feasible set intact, its solutions take the same form as classical OT, preserving full compatibility with downstream tools for transparency and explainability, such as \wax{}~\cite{Naumann2026}.

\smallskip

Through experiments on synthetic and real-world data, we demonstrate that our method learns OT solutions that improve upon those of classical OT. Specifically, we observe that our method is less susceptible to falling for implausible shortcuts in the input space. Our approach is sample-efficient, requiring only a limited number of carefully selected displacement instances to steer the OT solution toward more plausible outcomes, as we demonstrate in the case of retrieving migratory bird trajectories. Finally, our method also provides a foundation for higher performance in downstream applications, as we demonstrate in a domain adaptation use case.

\section{Related Work}
\label{sec:related-work}

\noindent Our work connects to two strategies for shaping OT solutions: constraining the \emph{coupling} (\cref{sec:rw-coupling}), or adapting the \emph{ground metric} (\cref{sec:rw-geom,sec:rw-gml}). We refer to~\cite{villaniOptimalTransportOld2008,Peyre2019} for a broader background on optimal transport in general.

\subsection{Constraining the Coupling}
\label{sec:rw-coupling}

One strategy alters the OT objective or feasible set so the coupling reflects structural priors. Spatial and graph-Laplacian regularizers encourage smooth, neighborhood-preserving plans~\cite{DBLP:journals/siamis/FerradansPPA14,courtyOptimalTransportDomain2017}. Schr\"odinger Bridges instead constrain the stochastic \emph{path}.
Theodoropoulos~\etal{}~\cite{DBLP:conf/iclr/TheodoropoulosK25}, for example, use pre-aligned pairs as feedback to guide unpaired transport, but alter both the objective and feasible set rather than the cost. Anchor-point methods restrict the coupling itself: Latent OT~\cite{linMakingTransportMore2021} forces transport through learned anchors for outlier-robust low-rank couplings; Sato~\etal{}~\cite{DBLP:journals/corr/abs-2002-01615} compare points via 1D pairwise-distance distributions; and Gu~\etal{}~\cite{DBLP:conf/nips/0005YZSX22} use keypoint-derived masks for heterogeneous domain adaptation.

In contrast, our \mname{} method imposes \emph{no} structural constraints on the plan, its path, its feasible set, or its regularization reference; it only modifies the cost matrix. This makes it interoperable with any cost-matrix-based OT solver.

\subsection{Hand-Designed and Robust Geometric Metrics}
\label{sec:rw-geom}

A complementary strategy replaces the Euclidean distance with one reflecting the data's intrinsic or worst-case geometry. Huguet~\etal{}~\cite{HuguetTZTWK23} compute Wasserstein distances under geodesic ground metrics via entropic OT, encoding manifold structure in the cost. Algebraically closer to ours, Paty~\&~Cuturi~\cite{patySubspaceRobustWasserstein2019} introduce Subspace Robust Wasserstein Distances, whose convex relaxation hinges on the second-moment matrix of displacements.

Such hand-designed or adversarially selected metrics are effective when the relevant geometry is known a priori or robustness is the goal, but they do not adapt to observed displacements. Both \emph{increase} transport costs, via topological constraints or adversarial projections, whereas \mname{} \emph{decreases} costs along directions empirically supported by observed displacements: it carves expressways rather than adding barriers (cf.~\cref{sec:geometric-interpretation}).

\subsection{Ground Metric Learning}
\label{sec:rw-gml}

Closest to our setting is \emph{ground metric learning}, where the cost is itself learned from data.
Several works learn a Mahalanobis cost for OT. Cuturi~\&~Avis~\cite{DBLP:journals/jmlr/CuturiA14} learn from class-labeled histograms via a subgradient procedure; Kerdoncuff~\etal{}~\cite{DBLP:conf/ijcai/KerdoncuffES20} establish a PCA--Wasserstein connection, and learn a ground metric to reduce target mismatch for domain adaptation; and Jawanpuria~\etal{}~\cite{DBLP:conf/icassp/JawanpuriaSMG25} jointly learn the OT plan and a latent ground metric via Riemannian alternating optimization.
Although the Mahalanobis form closely resembles ours, all three methods rely on label supervision or unpaired global statistics and require iterative optimization. \mname{} differs by deriving its metric in closed form from the second-order statistics of (possibly very few) observed source--target displacements, without label information and without alternating updates.
Lastly, we note that the recipe of inverting a second-moment matrix, as done by \mname{}, has roots in classical metric learning~\cite{DBLP:conf/nips/GoldbergerRHS04,DBLP:conf/icml/DavisKJSD07,DBLP:journals/jmlr/WeinbergerS09}, notably in Relevant Component Analysis (RCA)~\cite{DBLP:journals/jmlr/Bar-HillelHSW05}, which uses within-class covariances of equivalence-constrained ``chunklets''.
With \mname{}, we propose a method that brings this idea to OT.

A complementary line of work learns more expressive ground metrics, such as graph-geodesic, Riemannian, or neural, from observations of evolving densities or trajectories. Heitz~\etal{}~\cite{DBLP:journals/jmiv/HeitzBCCP21} parametrize the metric as graph edge weights and fit it by reconstructing intermediate density snapshots as Wasserstein geodesics: closest to ours in motivation, but requiring full intermediate marginals, graph-geodesic restrictions, and iterative diffusion-based optimization, whereas \mname{} only needs paired source--target observations and is closed-form. Scarvelis~\&~Solomon~\cite{DBLP:conf/iclr/Scarvelis023} learn a spatially-varying Riemannian metric from cross-sectional snapshots via a neural network; Pooladian~\etal{}~\cite{DBLP:conf/uai/PooladianDCA24} extend this to deterministic maps via a Lagrangian with obstacles and non-Euclidean geometries through amortized path optimization; and Kapusniak~\etal{}~\cite{DBLP:conf/nips/KapusniakPRZ0BB24} use data-dependent Riemannian metrics in metric flow matching to avoid straight-line shortcuts.
Closest to our supervision regime, OT-SI~\cite{DBLP:conf/iclr/LiuBZ20} learns a parametric OT cost from \emph{subset correspondences} by differentiating through Sinkhorn~\cite{cuturiSinkhornDistancesLightspeed2013} with pair matching being its extreme unit-sized case.
OT-SI and \mname{} share the goal of leveraging paired side information to shape the transport cost, but differ in approach: OT-SI fits a parametric cost via end-to-end gradient flow through the OT solver, trading analytical simplicity for flexibility, while \mname{} inverts a single moment statistic in closed form. The two methods, therefore, occupy complementary points, and we view them as compatible rather than competing.
Perrot~\etal{}~\cite{DBLP:conf/nips/PerrotCFH16} jointly learn a coupling and a parametric transport map from source--target pairs for out-of-sample transport.
Finally, Inverse OT~\cite{DBLP:journals/siamam/StuartW20,DBLP:conf/iclr/AndradePP24} addresses the reverse problem of inferring the cost from samples and a \emph{known} coupling, and Meta-OT~\cite{DBLP:conf/icml/AmosLCR23} amortizes solutions across related problems.

In contrast to these neural, graph-based, and iteratively-learned formulations, our \mname{} method requires no learned model, yields a closed-form ground metric, and directly leverages a small set of source--target pairings, whereas most prior methods rely on unlinked samples, full intermediate marginals, class labels, or end-to-end alignment objectives.
We also do not require additional information beyond the observed displacements to infer the adapted cost.
Hence, by modifying only the cost matrix, \mname{} further retains full compatibility with classical OT solvers.

\section{Problem Formulation and Proposed Method}
\label{sec:method}

\begin{figure}[t!]
    \centering
  \makebox[\textwidth][c]{
    \includegraphics[width=1.25\textwidth]{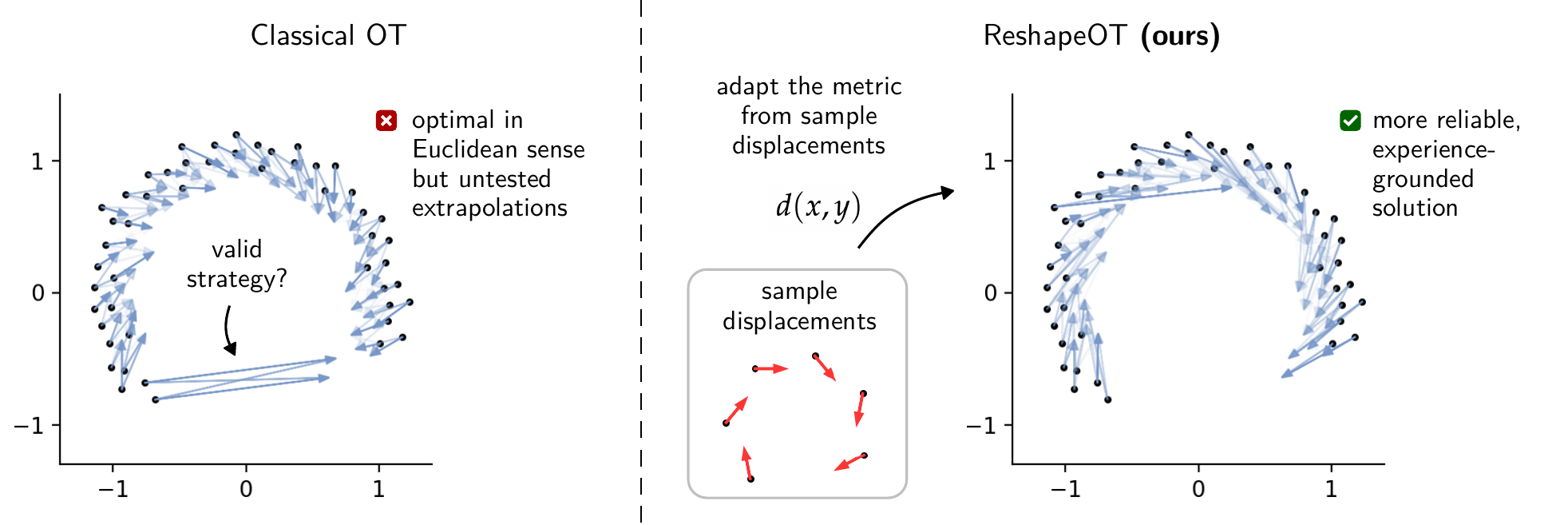}
    }
    \caption{Overview of our approach to influence the OT solution through ground-truth displacements. The classical OT solution (here with squared Euclidean costs) contains points that are spuriously transported across the manifold. In contrast, our proposed \mname{} method carves the original Euclidean distance into a new distance with lower associated costs along the displacements. This results in a more reliable, guided solution.}
    \label{fig:intro}
\end{figure}

\noindent We recall the classical OT formulation: given a source distribution $\mu$ and a target distribution $\nu$ on $\R^d$, the Kantorovich primal problem seeks a cost-minimizing coupling $\gamma\in\Gamma(\mu,\nu)$, a joint distribution on $\R^d \times \R^d$ with marginals $\mu$ and $\nu$
(cf.~\cite{villaniOptimalTransportOld2008,Peyre2019}):
\begin{align}
    \inf_{\gamma \in \Gamma(\mu,\nu)} \mathbb{E}_{(x,y) \sim \gamma} [c(x,y)] 
    \label{eq:ot}
\end{align}
where $c\colon \R^d \times \R^d \to \R_+$ is a cost function.
With finite samples from $\mu$ and $\nu$, $\gamma$ reduces to a non-negative matrix whose row and column sums match the empirical marginals.
This classical formulation has limitations in practice; in particular, we rarely know the true cost function, or the latter is impractical to compute. In practice, one often resorts to heuristic or technical cost functions, where a popular choice builds on the Euclidean ground metric, namely, $c(x,y) = \|x-y\|^2$, which is associated with the classical $\mathcal{W}_2$-Wasserstein transport problem.

The OT problem is a natural lens for studying distribution shift: $\mu$ and $\nu$ describe the data distribution before and after a shift, the coupling $\gamma$ encodes how mass is redistributed, and the cost function $c$ governs which movements are favored. The ground metric, therefore, directly determines what kinds of shifts the model can capture: a too rudimentary ground metric might invite the exploitation of shortcuts that are physically implausible, whereas an informed metric can reflect the true geometry of the distribution shift.

\medskip

We propose an alternative formulation to modeling distribution shifts using OT, where we assume access to an additional source of knowledge in the form of $\xp{N}$ pairings $(\xp{x_n},\xp{y_n})_{n=1}^\xp{N}$, which we refer to as \textit{sample displacements}.
These instances may be available for various reasons in practice. 
For example, they can arise from active tracking of a data subset (cf.~\cite{DBLP:conf/iclr/LiuBZ20}), from past transport phenomena (cf.~\cite{DBLP:conf/icassp/JawanpuriaSMG25}), or from manually defined or OT-generated virtual pairs in a controlled setting.
To incorporate such sample displacements, we now introduce our proposed method \textit{\mname{}}, which takes a classical OT formulation with Euclidean ground metric as a starting point and performs the following three conceptual steps:

\begin{description}
\item[Step 1 (Compute Displacement Statistics)]
The method starts by summarizing the available ground-truth displacements as a second-moment matrix $\Sigma$.
For the sake of generality, and to account for a broader set of usages, such as integrating knowledge from a general OT solution, we replace the sample of paired points by two sets of data points $(\xp{x_k})_{k=1}^\xp{M}$ and $(\xp{y_l})_{l=1}^\xp{N}$ linked through some coupling matrix $\xp{\gamma}$. (The special case of a displacement sample can be recovered by setting $\xp{M}=\xp{N}$ and limiting $\xp{\gamma}$ to permutation matrices.) Using the general formulation, we build $\Sigma$ as:
\begin{align}
\Sigma = \sum_{kl} \xp{\gamma_{kl}}(\xp{x_k}-\xp{y_l})(\xp{x_k}-\xp{y_l})^\top 
\label{eq:Sigma}
\end{align}
The leading eigenvectors of the matrix $\Sigma$ span the subspace of observed displacements. Simple global translations in input space are spanned well by a single leading eigenvector. More complex phenomena, like rotations, require more of them. 

\item[Step 2 (Ground Metric Redefinition)]
The method proceeds by defining a new cost function informed by the sample displacements, represented by the second-moment matrix $\Sigma$ calculated above. Specifically, we replace the original Euclidean metric with the Mahalanobis distance:
\begin{align}
d(x,y)
= \sqrt{(x-y)^\top (I + \eta\Sigma)^{-1} (x-y)} 
\label{eq:mahalanobis}
\end{align}
where $\eta \geq 0$ controls the degree to which the displacements shape the metric. At $\eta\!=\!0$, the expression reduces to the standard Euclidean distance. Increasing $\eta$ progressively flattens the metric along directions of observed displacement, making travel in those directions cheaper.

\item[Step 3 (Recalculation of the OT Solution)] As a last step, an OT solver of choice (e.g., \cref{eq:ot}) is applied to the costs derived from the new ground metric ($c(x,y) = d(x,y)^2$). This leads to a new solution for the coupling between the source and target that better aligns with the observed displacements, as illustrated in \cref{fig:intro}.
\end{description}
\noindent Technically, \mname{} is a recalculation of the classical OT problem under a new ground metric. It introduces no additional structural constraints on the transport plan and no significant computational overhead. The only new hyperparameter is $\eta$, which can be selected via cross-validation or set based on the desired strength of guidance. In \cref{sec:exp-ot-implementation-details}, we propose a heuristic to make setting it scale-invariant.

\subsection{Geometric Interpretation of \mname{} and Relation to $\mathcal{W}_2$}
\label{sec:geometric-interpretation}
\noindent In the following, we analyze the displacement-informed transport costs produced by \mname{}. Specifically, observing that $\Sigma$ is positive semi-definite, we can show that $d(x,y) \leq \|x-y\|$, with equality when $\eta\!=\!0$ (no influence) or $(x-y)\in\ker(\Sigma)$ (orthogonal to all given displacements). This relation extends to the distance between distributions, allowing us to show that \mname{} costs resulting from minimizing \cref{eq:ot} with the new costs are upper-bounded by the canonical $\mathcal{W}_2$-Wasserstein distance:
\begin{align}
    \inf_{\gamma \in \Gamma(\mu,\nu)} \mathbb{E}_{(x,y) \sim \gamma} [d(x,y)^2] \leq \W_2^2(\mu, \nu)
    \label{eq:expressway}
\end{align}
Here again, with equality when $\eta\!=\!0$ or when every displacement under $\W_2$ lies in $\ker(\Sigma)$, i.e.\ transport entirely orthogonal to the subspace of sample displacements.
The above inequality also shows that \mname{} strictly reduces the transport costs relative to $\W_2$, carving `expressways' through regions of high observed displacement. 
This stands in contrast to the more common use of geodesic distances (e.g.,~\cite{BonetDC25}), where the transport plan is instead reshaped by restricting the set of possible trajectories. As a corollary, we can show that \mname{} costs are equivalent to the squared $\mathcal{W}_2$-Wasserstein distance with a Euclidean ground metric applied to reparametrized marginals. Specifically, the cost equals $\W_2^2(T_{\#}\mu, T_{\#}\nu)$ with $T(x) = Wx$ and $W$ is any matrix satisfying $W^\top W = (I + \eta\Sigma)^{-1}$. Any such matrix $W$ (e.g., PCA, ZCA, or Cholesky) yields the same transport distance. This mapping rescales the displacement directions: along the $i$-th eigenvector of $\Sigma$, distances shrink by a factor $\sqrt{1/(1 + \eta\lambda_i)}$, i.e.\ directions of large observed displacement become cheap compared to directions that are not spanned by these observations.

\subsection{Extending \mname{} to Kernel Feature Spaces}
\label{sec:kernel-extension}

\noindent The ground metrics developed in \cref{sec:method} can globally reshape transport costs but have limits in capturing nonlinear transportation corridors found in real-world phenomena such as bird migration or blood circulation. To handle such cases, we extend \mname{} to kernel feature spaces~\cite{scholkopf1998nonlinear,DBLP:journals/tnn/MullerMRTS01,Schoelkopf2002book,DBLP:journals/pami/ZhangWN20}. Concretely, let $\kernel\colon\mathbb{R}^d\times\mathbb{R}^d \to \mathbb{R}$ be a positive semi-definite kernel, which subsumes some nonlinear mapping $x \mapsto \Phi(x)$ to some feature space, and in which distance calculations are equivalently expressible as kernel evaluations:
\begin{align*}
\|\Phi(x)-\Phi(y)\| = \sqrt{\kernel(x,x) - 2\kernel(x,y) + \kernel(y,y)}
\end{align*}
Considering now the \mname{} ground metric in \cref{eq:mahalanobis} but replacing any inner product or distance computation by the same operations in feature space, we can recover an expression of the distance that depends on the inputs only through kernel evaluations on the query pair $(x,y)$ and the individual source and targets from the sample $\boldsymbol{z} = (\xp{x_1}, \ldots, \xp{x_M}, \xp{y_1},\ldots,\xp{y_N})$. Specifically, we get the kernelized form of the \mname{} ground metric:
\begin{align}
d(\Phi(x),\Phi(y))
  = \sqrt{
  \begin{aligned}
  \kernel(x,x) &- 2\kernel(x,y) + \kernel(y,y)\\
  &- (\kernel(x,\boldsymbol{z}) - \kernel(y,\boldsymbol{z})) \, \Lambda \, (\kernel(\boldsymbol{z}, x) - \kernel(\boldsymbol{z}, y))
  \end{aligned}
  }
  \label{eq:kernel-distance}
\end{align}
with $\Lambda = \eta A^{1/2}(I + \eta A^{1/2} \kernel(\boldsymbol{z},\boldsymbol{z}) A^{1/2})^{-1} A^{1/2}$, where $A = (\diag(G\mathbf{1}) - G)$ is the weighted Laplacian of the coupling matrix $G = (0 , \xp{\gamma}; \xp{\gamma}^\top , 0)$, and where we have used $\mathcal{K}$ to also denote vectors or matrices of kernel computations. (A proof of \cref{eq:kernel-distance} is provided in Appendix~\ref{app:proof-kernelization}). Furthermore, we can verify that the inequalities presented in \cref{sec:geometric-interpretation} and their interpretations also hold for these kernelized distances and their associated Wasserstein distances.

\section{Experiments}
\label{sec:experiments}

\noindent In the following experiments, we evaluate \mname{}'s ability to recover ground-truth transport in synthetic and real-world settings.

\subsection{Datasets}
\label{sec:datasets}
Two of our three datasets are multivariate time series. \textit{Air Quality}~\cite{air_quality_360} records pollutant concentrations and meteorological variables in an Italian city at hourly intervals from March 2004 to February 2005. 
After preprocessing, it retains~9 features, with a median of~368 samples per shift in each source and target domain.
\textit{Appliances}~\cite{appliances_energy_prediction_374} tracks appliance energy consumption in a low-energy building via sensor measurements taken every 10 minutes over~4.5 months.
To match the hourly resolution of \textit{Air Quality}, we resample the records to hourly intervals using median aggregation.
After preprocessing, it retains~24 features, with a median of~133 samples per shift in each source and target domain.
From each dataset, we create subsets with temporal shifts of 1--6 hours by pairing a source (all samples across the full time series) with a target consisting of the same days shifted by the given delay, naturally inducing a ground-truth identity coupling.
Further details on both datasets and their preprocessing are provided in \supp{app:A-time-series-datasets}.
Lastly, we consider the synthetic \textit{Rotating Moons} problem, widely used in domain adaptation research~\cite{courtyOptimalTransportDomain2017}.
It consists of two nested 2D moon-shaped point clouds (150 samples each) representing binary classes; target domains are obtained by rotating these clouds counterclockwise by increasing angles (10$^\circ$\,--\,90$^\circ$), with larger rotations producing more severe distribution shifts (cf.\ \cref{fig:usecase_two,sec:use_case_da}).
The dataset generation script is provided in \supp{app:B-moons-generation}.

\subsection{Evaluation Metrics}
\label{sec:evaluation_metrics}

Our first evaluation assesses the quality of the coupling learned by \mname{} by comparing it against a ground-truth coupling derived from a fully tracked temporal phenomenon. To this end, we measure the mean error between ground-truth and predicted displacements.
In this setting, all methods have access to the complete set of pre-aligned point correspondences, i.e.~the ground-truth coupling is a normalized identity matrix. For each source point $x_k$, we apply the barycentric mapping induced by $\gamma$ to compute its predicted transport target $\hat{y}_k = T(x_k)$~\cite{DBLP:journals/siamis/FerradansPPA14,courtyOptimalTransportDomain2017}. We then evaluate the mean \textit{transport error} between the true target $y_k$ and its prediction $\hat{y}_k$, defined as $\mathbb{E}_k\big[\|y_k - \hat{y}_k\|\big]$. A perfect coupling, i.e.~one that exactly recovers the ground-truth pairing, yields $\hat{y}_k = y_k$ and a score of zero.

We perform an additional evaluation of \mname{} similar to the one originally proposed in~\cite{Naumann2026}, focusing on the transport model's ability to retrieve features relevant to the studied phenomena. This extrinsic evaluation enables a broader comparison with non-OT methods.
Given paired observations $(x_m,\, y_m)_{m=1}^M$, we obtain the ground-truth shift attributions as $R_i^\star = \mathbb{E}_m\big[ (x_{m,i} - y_{m,i})^2 \big]$.
We then split the original data into two unpaired subsets, apply each method to each subset to produce predicted attributions, and average the cosine similarities to the ground-truth attribution across the two splits.
For methods that output a coupling matrix $\gamma$, we compute the predicted per-feature attributions $R_i$ via \wax{}~\cite{Naumann2026} (with $p,q=2$), which decomposes the Wasserstein distance into non-negative feature contributions satisfying $\sum_i R_i = \mathcal{W}_2^2(\mu,\nu)$.
All other methods directly yield an attribution without the need for a post-hoc explainer.

\medskip

\subsection{\mname{} Implementation Details}
\label{sec:exp-ot-implementation-details}
We evaluate \mname{} using a linear kernel, $\kernel(x,y)\!=\!\langle x, y \rangle$, and using an RBF kernel, $\kernel(x,y)\!=\!\exp(-\lambda \|x-y\|^2)$.
In our experiments, we set $\eta = \alpha / \Tr(\Sigma)$, where $\alpha \geq 0$ is a dimensionless scale parameter. This normalization makes $\eta\Sigma$ scale-invariant with respect to the magnitude of the observed displacements, so that $\alpha$ has a consistent interpretation across datasets and kernels: $\alpha\!=\!0$ disables the displacement prior, while values of $\alpha \gg 1$ correspond to strong guidance along the directions captured by $\Sigma$.
While this removes the need to precisely set $\eta$, the $\alpha$ parameter still remains problem-dependent. If a task requires stricter guidance, $\alpha$ needs to be increased accordingly.

For the time series datasets, we select $\xp{N}\!=\!5$ random pairs from the ground-truth coupling to build the displacement and remove them from the remaining data at each experiment iteration, and use $\lambda\!=\!0.0005$ for the RBF kernel and $\alpha\!=\!10^2$ for the regularization.
For the \textit{Rotating Moons} dataset, we assemble a sample of displacements of size $\xp{N}\!=\!40$, where each instance is selected randomly. When applying \mname{} to this dataset, we use an RBF kernel and $\alpha\!=\!10^4$ to ensure an efficient capture of the information contained in the sample displacements, and use and integration strength of $\lambda\!=\!2$.
These hyperparameters are based on the best-performing results of a hyperparameter sweep provided in \supp{app:C-hyperparameter-selection}.
Furthermore, we show ablations of \mname{} that replace the coupling matrix $\xp{\gamma}$ by a randomly permuted version $\xp{\gamma_\mathrm{rng}}$, allowing us to test if performance improvement derives from exploiting the exact linking of sources or targets in the sample or their overall variance.

Note that \wax{} assumes a Minkowski metric, hence we use \mname{} only as a regularizer here to find $\gamma$ based on the reshaped ground metric, but evaluate the feature attributions using this coupling and the squared Euclidean cost function (corresponding to \wax{} with $p,q\!=\!2$).
We do not evaluate the \textit{Rotating Moons} dataset on the cosine metric, as it is too low-dimensional to yield meaningful insights.

\subsection{Results}
The results in \cref{tab:results_cos,tab:results_coupling} show averages over 20 random trials: for each time-pair and rotation, we repeated the experiment~20 times with a different random seed, thereby changing the splits and selected sample displacements each time.
Furthermore, \cref{tab:hp_nexp} presents an ablation over the number of utilized displacement pairs ($\xp{N}$) for the time series datasets, where $\xp{N}\!=\!0$ corresponds to \emd{} in RBF kernel space.

In \cref{tab:results_coupling}, we show the results for how well the induced barycentric transport map aligns with the ground-truth coupling.
We see that \emd{}, the unregularized solution of solving \cref{eq:ot} with squared Euclidean cost, performs worse than \mname{}, but better than Sinkhorn, which is an entropy-regularized variant of \cref{eq:ot}~\cite{cuturiSinkhornDistancesLightspeed2013} and yields a stochastic transport plan.
This discrepancy between \emd{} and Sinkhorn is likely due to the ground-truth transport being non-stochastic, which gives an advantage to methods that yield a permutation-like coupling (\emd{} and \mname{}).
Overall, \mname{} outperforms all other baselines by a margin, clearly demonstrating the guidance effect of using displacement information.
In the \textit{Rotating Moons} dataset, the random sample displacement ablations fail to provide useful information, as the raw second-moment statistics without informative guidance are insufficient to capture the nature of the rotation shift.
In this case, the data manifold has a more complex shape, and the locality of insightful displacements is required to accurately guide OT to transport along this manifold.
The experimental results in \cref{sec:use_case_da} emphasize this behavior in the case of a domain adaptation task.
Furthermore, the more complex manifold shapes here require a nonlinear kernel for accurate modeling.
Hence, we see that \mname{} with an RBF kernel outperforms the linear kernel, achieving nearly perfect alignment.

\begin{table}[H]
  \centering
  \caption{%
  Mean transport errors\,$\downarrow$ at different shift levels.
  The ablations of \mname{} use the same parametrization as their main counterparts.
  The values are the mean $\pm$ std.\ deviation. \textbf{Bold} = best, \textit{italic} = 2nd best, ranked within row.
  }%
  \label{tab:results_coupling}
    \resizebox{\textwidth}{!}{%
  \LARGE
  \begin{tabular}{llccccccc}
    \toprule
     &  & \multicolumn{3}{c}{Baselines} & \multicolumn{2}{c}{\textbf{Ours}} & \multicolumn{2}{c}{Ablations} \\
    \cmidrule(lr){3-5} \cmidrule(lr){6-7} \cmidrule(lr){8-9}
    Dataset & Shift & Classical & Sinkhorn & Sinkhorn & \mname{} & \mname{} & \mname{} & \mname{} \\
     &  & OT & {\Large($\varepsilon\!=\!0.05$)} & {\Large($\varepsilon\!=\!0.1$)} & {\Large(Linear)} & {\Large(RBF)} & {\Large(Linear, $\xp{\gamma_\mathrm{rng}}$)} & {\Large(RBF, $\xp{\gamma_\mathrm{rng}}$)} \\
    \midrule
    \multirow{6}{*}{Air Quality} & 1h & $0.44_{\pm 0.17}$ & $1.08_{\pm 0.16}$ & $1.32_{\pm 0.16}$ & $\mathit{0.20_{\pm 0.11}}$ & $\mathbf{0.20_{\pm 0.11}}$ & $0.36_{\pm 0.15}$ & $0.35_{\pm 0.16}$ \\
     & 2h & $0.78_{\pm 0.27}$ & $1.22_{\pm 0.21}$ & $1.42_{\pm 0.20}$ & $\mathit{0.56_{\pm 0.23}}$ & $\mathbf{0.55_{\pm 0.23}}$ & $0.73_{\pm 0.27}$ & $0.73_{\pm 0.28}$ \\
     & 3h & $0.99_{\pm 0.34}$ & $1.32_{\pm 0.24}$ & $1.49_{\pm 0.22}$ & $\mathit{0.80_{\pm 0.29}}$ & $\mathbf{0.80_{\pm 0.29}}$ & $0.96_{\pm 0.33}$ & $0.96_{\pm 0.34}$ \\
     & 4h & $1.14_{\pm 0.39}$ & $1.39_{\pm 0.26}$ & $1.54_{\pm 0.23}$ & $\mathit{0.96_{\pm 0.30}}$ & $\mathbf{0.95_{\pm 0.30}}$ & $1.12_{\pm 0.34}$ & $1.11_{\pm 0.34}$ \\
     & 5h & $1.26_{\pm 0.40}$ & $1.46_{\pm 0.26}$ & $1.58_{\pm 0.22}$ & $\mathit{1.07_{\pm 0.29}}$ & $\mathbf{1.07_{\pm 0.30}}$ & $1.24_{\pm 0.34}$ & $1.23_{\pm 0.35}$ \\
     & 6h & $1.36_{\pm 0.37}$ & $1.51_{\pm 0.25}$ & $1.62_{\pm 0.21}$ & $\mathit{1.18_{\pm 0.26}}$ & $\mathbf{1.17_{\pm 0.26}}$ & $1.33_{\pm 0.33}$ & $1.32_{\pm 0.33}$ \\
    \midrule
    \multirow{6}{*}{Appliances} & 1h & $0.25_{\pm 0.18}$ & $1.84_{\pm 0.11}$ & $2.35_{\pm 0.10}$ & $\mathit{0.06_{\pm 0.11}}$ & $\mathbf{0.06_{\pm 0.11}}$ & $0.21_{\pm 0.22}$ & $0.18_{\pm 0.20}$ \\
     & 2h & $0.76_{\pm 0.31}$ & $2.02_{\pm 0.15}$ & $2.45_{\pm 0.12}$ & $\mathbf{0.24_{\pm 0.25}}$ & $\mathit{0.25_{\pm 0.25}}$ & $0.65_{\pm 0.44}$ & $0.58_{\pm 0.39}$ \\
     & 3h & $1.11_{\pm 0.40}$ & $2.13_{\pm 0.16}$ & $2.51_{\pm 0.13}$ & $\mathbf{0.45_{\pm 0.37}}$ & $\mathit{0.46_{\pm 0.37}}$ & $1.04_{\pm 0.57}$ & $0.93_{\pm 0.50}$ \\
     & 4h & $1.40_{\pm 0.40}$ & $2.21_{\pm 0.16}$ & $2.56_{\pm 0.14}$ & $\mathbf{0.70_{\pm 0.45}}$ & $\mathit{0.72_{\pm 0.44}}$ & $1.40_{\pm 0.61}$ & $1.28_{\pm 0.54}$ \\
     & 5h & $1.63_{\pm 0.35}$ & $2.26_{\pm 0.16}$ & $2.59_{\pm 0.13}$ & $\mathbf{0.96_{\pm 0.50}}$ & $\mathit{0.98_{\pm 0.48}}$ & $1.70_{\pm 0.58}$ & $1.59_{\pm 0.53}$ \\
     & 6h & $1.82_{\pm 0.30}$ & $2.31_{\pm 0.15}$ & $2.62_{\pm 0.13}$ & $\mathbf{1.26_{\pm 0.49}}$ & $\mathit{1.27_{\pm 0.47}}$ & $1.98_{\pm 0.51}$ & $1.86_{\pm 0.48}$ \\
    \midrule
    \multirow{5}{*}{Rotating Moons} & 10$^\circ$ & $\mathit{0.04_{\pm 0.01}}$ & $0.22_{\pm 0.01}$ & $0.34_{\pm 0.01}$ & $0.04_{\pm 0.01}$ & $\mathbf{0.04_{\pm 0.01}}$ & $0.15_{\pm 0.03}$ & $0.31_{\pm 0.05}$ \\
     & 30$^\circ$ & $0.34_{\pm 0.01}$ & $0.36_{\pm 0.00}$ & $0.43_{\pm 0.00}$ & $\mathit{0.24_{\pm 0.01}}$ & $\mathbf{0.09_{\pm 0.02}}$ & $0.56_{\pm 0.07}$ & $0.88_{\pm 0.09}$ \\
     & 50$^\circ$ & $0.69_{\pm 0.01}$ & $0.62_{\pm 0.01}$ & $0.60_{\pm 0.01}$ & $\mathit{0.51_{\pm 0.02}}$ & $\mathbf{0.06_{\pm 0.01}}$ & $0.91_{\pm 0.05}$ & $1.04_{\pm 0.15}$ \\
     & 70$^\circ$ & $1.03_{\pm 0.01}$ & $0.92_{\pm 0.01}$ & $0.85_{\pm 0.01}$ & $\mathit{0.78_{\pm 0.03}}$ & $\mathbf{0.08_{\pm 0.03}}$ & $1.15_{\pm 0.08}$ & $1.16_{\pm 0.10}$ \\
     & 90$^\circ$ & $1.33_{\pm 0.01}$ & $1.21_{\pm 0.01}$ & $1.11_{\pm 0.01}$ & $\mathit{1.06_{\pm 0.03}}$ & $\mathbf{0.10_{\pm 0.02}}$ & $1.37_{\pm 0.08}$ & $1.18_{\pm 0.11}$ \\
    \bottomrule
  \end{tabular}}
\end{table}

\medskip

In \cref{tab:results_cos} we show the feature attribution evaluation results.
Here, \emd{} and Sinkhorn also use \wax{} to compute feature attributions.
Additionally, we added the \textit{Constant Shift} ($R_i=1$) and \textit{Mean Shift} ($R_i = (\mathbb{E}[x] - \mathbb{E}[y])^2_i$) baselines from~\cite{Naumann2026}, which directly provide such attributions without an intermediate coupling step.
We see that \mname{} achieves the highest attribution similarities across all methods, especially improving on the unguided OT formulations \emd{} and Sinkhorn.
There is no clear difference between using a linear or an RBF kernel for these datasets, as the transport processes are likely driven by a clear trend rather than shifting along a complex manifold.
Relative performance across methods generally aligns with the results from \cref{tab:results_coupling}, with the exception that Sinkhorn outperforms \emd{} in the \textit{Appliances} dataset. Since this evaluation abstracts away from the exact coupling, the Sinkhorn regularization helps estimate more robust attributions that reflect the underlying shift trend.
Notably, \mname{} with random sample displacements still outperforms \emd{} and Sinkhorn in the \textit{Air Quality} dataset. This is because the random coupling preserves global second-moment structure, information about the relative spread of source and target that \emd{} must infer entirely from the marginals. When the distribution shift has a dominant constant trend, this global shape information is sufficient to improve the transport.
As in \cref{tab:results_coupling}, Sinkhorn outperforms \emd{} in the \textit{Appliances} dataset, and now also beats the \mname{} ablations.

\begin{table}[H]
  \centering
  \caption{%
  Cosine similarities\,$\uparrow$ between the predicted and ground-truth shift attributions.
  The ablations of \mname{} use the same parametrization as their main counterparts.
  The values are the mean $\pm$ std.\ deviation. \textbf{Bold} = best, \textit{italic} = 2nd best, ranked within row.
  }%
  \label{tab:results_cos}
    \resizebox{\textwidth}{!}{%
  \LARGE
  \begin{tabular}{llcccccccc}
    \toprule
     &  & \multicolumn{4}{c}{Baselines} & \multicolumn{2}{c}{\textbf{Ours}} & \multicolumn{2}{c}{Ablations} \\
    \cmidrule(lr){3-6} \cmidrule(lr){7-8} \cmidrule(lr){9-10}
    Dataset & Shift & Constant & Mean & Classical & Sinkhorn & \mname{} & \mname{} & \mname{} & \mname{} \\
     &  & Shift & Shift & OT & {\Large($\varepsilon\!=\!0.05$)} & {\Large(Linear)} & {\Large(RBF)} & {\Large(Linear, $\xp{\gamma_\mathrm{rng}}$)} & {\Large(RBF, $\xp{\gamma_\mathrm{rng}}$)} \\
    \midrule
    \multirow{6}{*}{Air Quality} & 1h & $0.81_{\pm 0.08}$ & $0.80_{\pm 0.22}$ & $0.87_{\pm 0.07}$ & $0.79_{\pm 0.05}$ & $\mathbf{0.97_{\pm 0.03}}$ & $\mathit{0.97_{\pm 0.03}}$ & $0.87_{\pm 0.12}$ & $0.87_{\pm 0.11}$ \\
     & 2h & $0.82_{\pm 0.08}$ & $0.87_{\pm 0.19}$ & $0.92_{\pm 0.07}$ & $0.84_{\pm 0.06}$ & $\mathbf{0.98_{\pm 0.03}}$ & $\mathit{0.98_{\pm 0.03}}$ & $0.93_{\pm 0.09}$ & $0.93_{\pm 0.09}$ \\
     & 3h & $0.83_{\pm 0.07}$ & $0.89_{\pm 0.15}$ & $0.94_{\pm 0.07}$ & $0.88_{\pm 0.06}$ & $\mathbf{0.98_{\pm 0.03}}$ & $\mathit{0.98_{\pm 0.03}}$ & $0.95_{\pm 0.07}$ & $0.95_{\pm 0.07}$ \\
     & 4h & $0.85_{\pm 0.06}$ & $0.88_{\pm 0.17}$ & $0.95_{\pm 0.07}$ & $0.90_{\pm 0.06}$ & $\mathbf{0.98_{\pm 0.02}}$ & $\mathit{0.98_{\pm 0.02}}$ & $0.95_{\pm 0.08}$ & $0.95_{\pm 0.08}$ \\
     & 5h & $0.86_{\pm 0.05}$ & $0.91_{\pm 0.11}$ & $0.95_{\pm 0.05}$ & $0.91_{\pm 0.05}$ & $\mathbf{0.98_{\pm 0.02}}$ & $\mathit{0.98_{\pm 0.02}}$ & $0.96_{\pm 0.05}$ & $0.96_{\pm 0.05}$ \\
     & 6h & $0.88_{\pm 0.05}$ & $0.91_{\pm 0.10}$ & $0.95_{\pm 0.05}$ & $0.92_{\pm 0.04}$ & $\mathbf{0.98_{\pm 0.03}}$ & $\mathit{0.98_{\pm 0.03}}$ & $0.96_{\pm 0.05}$ & $0.96_{\pm 0.05}$ \\
    \midrule
    \multirow{6}{*}{Appliances} & 1h & $0.32_{\pm 0.06}$ & $0.46_{\pm 0.30}$ & $0.69_{\pm 0.09}$ & $0.74_{\pm 0.07}$ & $\mathbf{0.89_{\pm 0.10}}$ & $\mathit{0.89_{\pm 0.10}}$ & $0.72_{\pm 0.16}$ & $0.70_{\pm 0.15}$ \\
     & 2h & $0.34_{\pm 0.07}$ & $0.45_{\pm 0.29}$ & $0.70_{\pm 0.11}$ & $0.76_{\pm 0.07}$ & $\mathbf{0.91_{\pm 0.09}}$ & $\mathit{0.90_{\pm 0.09}}$ & $0.75_{\pm 0.15}$ & $0.73_{\pm 0.15}$ \\
     & 3h & $0.37_{\pm 0.08}$ & $0.45_{\pm 0.26}$ & $0.73_{\pm 0.12}$ & $0.79_{\pm 0.06}$ & $\mathbf{0.91_{\pm 0.08}}$ & $\mathit{0.91_{\pm 0.08}}$ & $0.77_{\pm 0.14}$ & $0.75_{\pm 0.14}$ \\
     & 4h & $0.40_{\pm 0.09}$ & $0.45_{\pm 0.25}$ & $0.75_{\pm 0.11}$ & $0.81_{\pm 0.06}$ & $\mathbf{0.92_{\pm 0.07}}$ & $\mathit{0.91_{\pm 0.07}}$ & $0.79_{\pm 0.13}$ & $0.77_{\pm 0.12}$ \\
     & 5h & $0.43_{\pm 0.09}$ & $0.47_{\pm 0.23}$ & $0.77_{\pm 0.11}$ & $0.82_{\pm 0.05}$ & $\mathbf{0.92_{\pm 0.06}}$ & $\mathit{0.92_{\pm 0.06}}$ & $0.80_{\pm 0.12}$ & $0.79_{\pm 0.12}$ \\
     & 6h & $0.46_{\pm 0.09}$ & $0.48_{\pm 0.22}$ & $0.79_{\pm 0.11}$ & $0.84_{\pm 0.05}$ & $\mathbf{0.93_{\pm 0.06}}$ & $\mathit{0.92_{\pm 0.06}}$ & $0.82_{\pm 0.11}$ & $0.81_{\pm 0.11}$ \\
    \bottomrule
  \end{tabular}}
\end{table}

Finally, \cref{tab:hp_nexp} demonstrates the efficiency of \mname{} with respect to the number of provided displacement pairs in terms of the transport error.
The $\xp{N}\!=\!0$ case represents \emd{} in RBF kernel space without any displacement pairs, and we clearly see that the kernel alone is insufficient to improve its performance. The results are essentially identical to those of using \emd{} in input space as shown in \cref{tab:results_coupling}.
We present an ablation study across different kernel parameter values, $\lambda$, in \supp{app:D-classical-ot-rbf}.
Interestingly, including only one displacement pair as guidance already allows \mname{} to outperform the baselines. This highlights the critical information about the shift contained in a single correctly coupled pair, as well as \mname{}'s ability to leverage such sparse data. As expected, performance improves further as we gradually increase the number of observed displacements.

\begin{table}[htb]
  \centering
  \caption{Effect of $\xp{N}$ on transport error\,$\downarrow$ for \mname{} (RBF, $\lambda\!=\!0.0005$, $\alpha\!=\!10^2$). The values are the mean $\pm$ std.\ deviation. \textbf{Bold} = best, \textit{italic} = 2nd best, ranked within row.}
  \label{tab:hp_nexp}
    \resizebox{\textwidth}{!}{%
  \scriptsize
  \begin{tabular}{llcccccc}
    \toprule
    Dataset & Shift & $\xp{N}=0$ & $\xp{N}=1$ & $\xp{N}=3$ & $\xp{N}=5$ & $\xp{N}=10$ & $\xp{N}=20$ \\
    \midrule
    \multirow{6}{*}{Air Quality} & 1h & $0.45_{\pm 0.169}$ & $0.34_{\pm 0.179}$ & $0.25_{\pm 0.153}$ & $0.20_{\pm 0.107}$ & $0.16_{\pm 0.080}$ & $\mathbf{0.13_{\pm 0.071}}$ \\
     & 2h & $0.78_{\pm 0.265}$ & $0.71_{\pm 0.285}$ & $0.62_{\pm 0.269}$ & $0.55_{\pm 0.230}$ & $0.49_{\pm 0.216}$ & $\mathbf{0.45_{\pm 0.197}}$ \\
     & 3h & $0.99_{\pm 0.337}$ & $0.94_{\pm 0.333}$ & $0.86_{\pm 0.315}$ & $0.80_{\pm 0.294}$ & $0.74_{\pm 0.281}$ & $\mathbf{0.70_{\pm 0.269}}$ \\
     & 4h & $1.14_{\pm 0.392}$ & $1.08_{\pm 0.337}$ & $1.01_{\pm 0.323}$ & $0.95_{\pm 0.304}$ & $0.89_{\pm 0.285}$ & $\mathbf{0.85_{\pm 0.273}}$ \\
     & 5h & $1.26_{\pm 0.398}$ & $1.20_{\pm 0.337}$ & $1.12_{\pm 0.308}$ & $1.07_{\pm 0.295}$ & $1.02_{\pm 0.282}$ & $\mathbf{0.98_{\pm 0.278}}$ \\
     & 6h & $1.36_{\pm 0.373}$ & $1.31_{\pm 0.323}$ & $1.22_{\pm 0.273}$ & $1.17_{\pm 0.264}$ & $1.12_{\pm 0.260}$ & $\mathbf{1.09_{\pm 0.262}}$ \\
    \midrule
    \multirow{6}{*}{Appliances} & 1h & $0.26_{\pm 0.172}$ & $0.20_{\pm 0.178}$ & $0.10_{\pm 0.128}$ & $0.06_{\pm 0.108}$ & $0.03_{\pm 0.063}$ & $\mathbf{0.01_{\pm 0.041}}$ \\
     & 2h & $0.77_{\pm 0.314}$ & $0.62_{\pm 0.353}$ & $0.37_{\pm 0.302}$ & $0.25_{\pm 0.246}$ & $0.11_{\pm 0.147}$ & $\mathbf{0.05_{\pm 0.083}}$ \\
     & 3h & $1.12_{\pm 0.400}$ & $0.96_{\pm 0.443}$ & $0.63_{\pm 0.426}$ & $0.46_{\pm 0.371}$ & $0.25_{\pm 0.235}$ & $\mathbf{0.12_{\pm 0.137}}$ \\
     & 4h & $1.41_{\pm 0.406}$ & $1.27_{\pm 0.471}$ & $0.93_{\pm 0.485}$ & $0.72_{\pm 0.439}$ & $0.43_{\pm 0.317}$ & $\mathbf{0.23_{\pm 0.187}}$ \\
     & 5h & $1.64_{\pm 0.341}$ & $1.51_{\pm 0.448}$ & $1.20_{\pm 0.490}$ & $0.98_{\pm 0.480}$ & $0.65_{\pm 0.398}$ & $\mathbf{0.38_{\pm 0.283}}$ \\
     & 6h & $1.84_{\pm 0.291}$ & $1.72_{\pm 0.401}$ & $1.47_{\pm 0.459}$ & $1.27_{\pm 0.467}$ & $0.94_{\pm 0.437}$ & $\mathbf{0.63_{\pm 0.374}}$ \\
    \bottomrule
  \end{tabular}}
\end{table}

\medskip

In summary, our results show that displacement guidance significantly improves OT's ability to capture the true transport phenomena, even with only limited access.
Furthermore, we showed that the choice of kernel plays an important role in the modeling of the feature space and helps account for nonlinear shifts. Specifically, the kernel does not improve the transport solution in isolation; rather, it serves as a mechanism for integrating knowledge from sample displacements more effectively.

\section{Use Case 1: Refining Transport Models with Longitudinal Tracking}
\label{sec:use_case_migration}

\noindent Predicting and understanding the motion of populations is a central research question with relevance in a variety of scientific domains, such as epidemiology, wildlife monitoring, urban mobility, and more. This task becomes particularly challenging when population motion is only partially observable, which may arise from the technical infeasibility of large-scale longitudinal tracking or from its explicit avoidance to address privacy requirements. In practice, available data may consist of a bulk of unlinked data augmented by a few linked instances, e.g., collected through a specific study.

\begin{figure}[b!]
    \centering
    \makebox[\textwidth][c]{
  \includegraphics[width=1.2\linewidth,trim={0.2cm 0.2cm 0.2cm 0.2cm},clip]{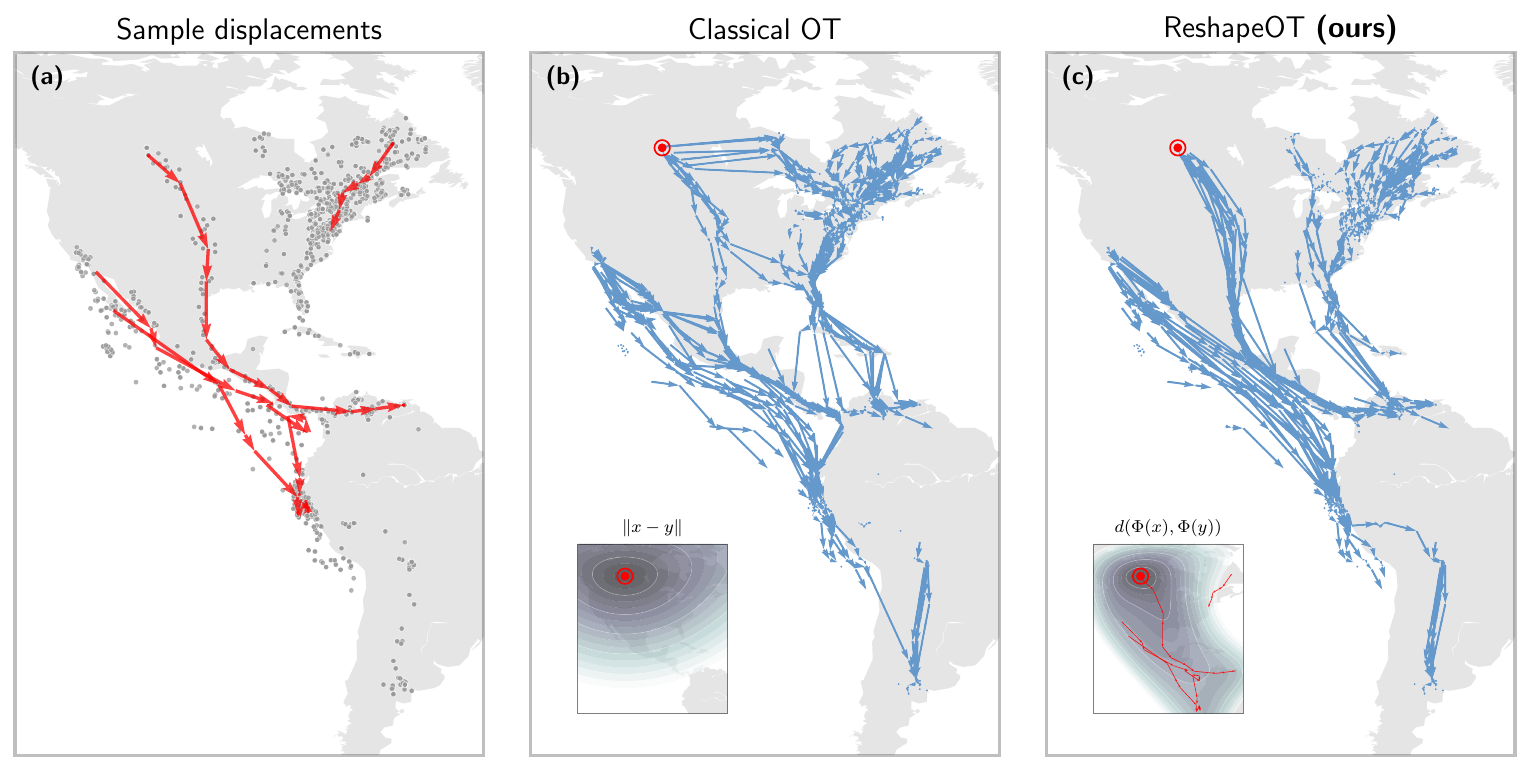}
  }
    \caption{%
    (a)~Training observations (gray) from multiple bird tracking studies and manually selected displacements (red).
    (b)~Coupling of \emd{} from a squared Euclidean cost matrix built on Cartesian coordinates.
    (c)~Coupling of \mname{} with RBF kernel ($\lambda\!=\!1$, $\alpha\!=\!10^3$) on Cartesian coordinates.
    The insets show the induced square-root (for better visibility) cost fields of \emd{} and \mname{} for a given reference point relative to a grid of targets.
    }%
    \label{fig:usecase_one}
\end{figure}

To demonstrate \mname{} on a real-world geospatial transport problem, we consider bird migration and the retrieval of migratory routes.
In this scenario, we have extensive population-level information, but we do not know individuals' specific travel routes.
A small number of birds, however, have GPS trackers that provide ground-truth displacement data along their flight paths. We use this setting to demonstrate \mname{}'s utility in leveraging such sparse individual-tracking data to recover population transport dynamics more accurately than using marginal data alone.
While a similar experimental setup was adopted by Scarvelis~\&~Solomon~\cite{DBLP:conf/iclr/Scarvelis023}, we focus specifically on integrating individual bird-tracking data rather than population-level observations.

We aggregated bird-tracking data from multiple studies from the MoveBank repository~\cite{movebank}, covering various migratory species observed across North America (\supp{app:E-bird-data} provides a list of the individual studies). For each bird and year, we compute median geographic positions (latitude and longitude) per individual and calendar week over a period of~13 weeks (mid-September to mid-December), representing autumn migration. Consecutive weeks serve as source and target distributions, representing population snapshots before and after a migration phase.
Since this splitting creates overlapping domains, we added small Gaussian noise ($\sigma\!=\!10^{-4}$) to avoid identical points and to break potential ties.
We remove outlier displacements beyond the 99th percentile to exclude, e.g., GPS artifacts.
The dataset features are latitude and longitude information, which we transformed into 3D Cartesian coordinates so that the Euclidean distance corresponds to the chordal distance between points on the sphere, a close approximation of the great-circle distance at the spatial scales of weekly displacements.

To construct the ground-truth displacement instances used as guidance, we manually selected~4 representative, complete bird trajectories (cf.\ \supp{app:E-bird-data}) spanning the entire assessed time frame from the available studies and removed those points from the training data.
This preprocessing created \num{2266} untracked source and target samples, and $\xp{N}\!=\!47$ ground-truth displacements, as shown in \cref{fig:usecase_one}a. To integrate the sample displacements well, in particular to account for local variations in the direction of sample trajectories, we opt for the RBF kernel variant of \mname{}.
As a baseline, we use \emd{} with a squared Euclidean cost matrix built on the Cartesian coordinates.

\Cref{fig:usecase_one} shows the coupling solutions for \emd{} in panel~(b), and \mname{} in~(c). 
Classical OT scatters mass broadly, including substantial east--west transport that is physically implausible for autumn migration, which proceeds primarily north-to-south. The inset in~(b) reveals the cause: a Euclidean ground metric induces a rotationally-symmetric field around any reference point (distorted only visually by the 2D map projection), so every geographic direction is equally cheap, and the optimal coupling has no preference for the dominant migration corridors.
By incorporating ground-truth migration routes from only four birds (shown as red arrows in~(a)), \mname{} reconfigures the cost function to favor movement along the observed migration patterns, concretely, prioritizing the north--south corridor over spurious east--west travel. The inset in~(c) shows the effect on the ground metric: the field is now stretched along the observed displacements, making north--south movement cheap and east--west travel comparatively expensive.
This reshaping leads to more consistent and plausible trajectories that remain consistent with their local trends (i.e.~species or swarms travel along similar routes).

\medskip

In summary, this use case shows how \mname{}, specifically its kernel variant, efficiently integrates real-world displacements to yield more plausible couplings. By using data from only four tracked trajectories, \mname{} recovers the dominant north--south migration corridor while avoiding spurious east--west shortcuts.

\section{Use Case 2: Domain Adaptation Under Manifold Rotation}
\label{sec:use_case_da}
\noindent In our second use case, we apply \mname{} to a task, in which a classifier is trained on a labeled source domain and must adapt to an unlabeled but related target domain~\cite{DBLP:conf/icml/GaninL15}.
OT-based domain adaptation addresses this by, e.g., transporting source points (together with their labels) to the target distribution via the barycentric mapping induced by the coupling~\cite{DBLP:journals/siamis/FerradansPPA14,courtyOptimalTransportDomain2017}, and training the classifier on the transported data.

We revisit the \textit{Rotating Moons} scenario from \cref{sec:datasets}, this time as a synthetic domain-adaptation benchmark: a binary classification task with two interleaved moon-shaped classes, where the target distribution is obtained by rotating the source~\cite{courtyOptimalTransportDomain2017}.
We additionally sample a held-out test set of \num{1000} points from the rotated target distribution.
The \textit{Setup} column in \cref{fig:usecase_two} shows a concrete example at 40$^\circ$.
This rotating two-moons setup mirrors the synthetic benchmark of Liu \etal{}~\cite{DBLP:conf/iclr/LiuBZ20}, where pair matchings likewise serve as side information to probe whether OT recovers the rotation rather than collapsing onto a cross-class shortcut between adjacent moon edges.
We compare \mname{} to \emd{}, Sinkhorn, and Laplace-OT~\cite{courtyOptimalTransportDomain2017}.
As in~\cref{sec:experiments}, we use an RBF kernel ($\lambda\!=\!2$) with $\alpha\!=\!10^4$ for regularization of $\Sigma$.
For each method and trial, we transport the labeled source to the target via barycentric mapping, train an SVM on the transported data with hyperparameters selected via grid search, and report classification error on the held-out test set.
We repeat this protocol over 20 random seeds.

\begin{figure}[t]
    \centering
    \makebox[\textwidth][c]{
    \includegraphics[width=1.2\linewidth,trim={0.2cm 0.2cm 0.2cm 0.2cm},clip]{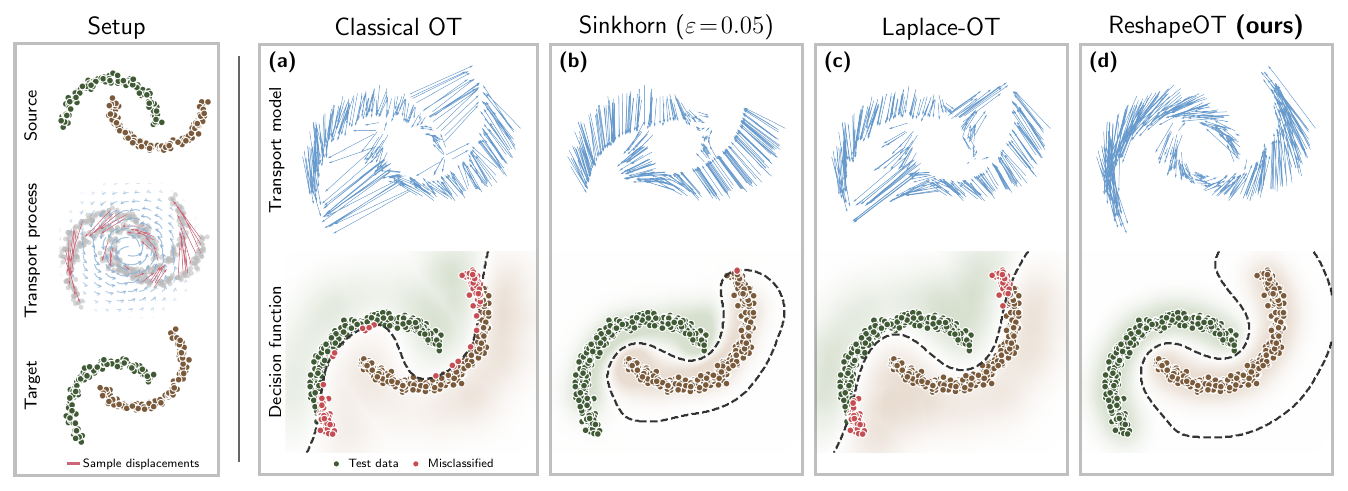}}
    \caption{Comparison of different OT-based domain adaptation methods on the \textit{Rotating Moons} task at $40^\circ$ rotation. Arrows denote the transport to the barycentric mapping, as determined by the coupling for each method and setting. The experiment uses $150$ samples per class and domain, and $\xp{N}\!=\!40$ randomly selected ground-truth displacements.}
    \label{fig:usecase_two}
\end{figure}

\medskip

In \cref{fig:usecase_two} we exemplify the transport models of the baselines and \mname{} for one of the trials at 40$^\circ$ rotation, together with the learned SVM decision boundaries and predictions.
It is apparent that standard OT methods resort to spurious shortcuts in pursuit of the least-action principle.
The entropic regularization helps improve this issue, but it also alters the shape of the transported distribution.

Classification errors at various degrees of target rotation are shown in \cref{fig:moons_da_results}, with increasing rotation angles corresponding to increasingly difficult domain-adaptation problems. A more exhaustive table of results is provided in \supp{app:F-further-da-results}, including additional baselines and ablations.
The results clearly show that \mname{} outperforms all other methods in the most challenging settings.
The guidance provided by the ground-truth displacements is a good model of the manifold transformation, yielding robust couplings across all levels of rotation.
Notably, Sinkhorn with $\varepsilon\!=\!0.05$ is a strong competitor in the domain adaptation setting, although it performed poorly in our quantitative experiments (see \cref{tab:results_coupling}).
Entropic smoothing prevents Sinkhorn from recovering exact pointwise correspondences. But that same smoothing also averages out cross-class shortcuts in the barycentric mapping, preserving the dominant rotation direction (cf.\ \cref{fig:usecase_two}) sufficiently for the SVM to fit a correct decision boundary.
Note that a linear kernel fails in this setting as it cannot account for the rotation geometry (cf.\ the results in \supp{app:F-further-da-results}).

\begin{figure}[H]
    \centering
    \includegraphics[width=\linewidth]{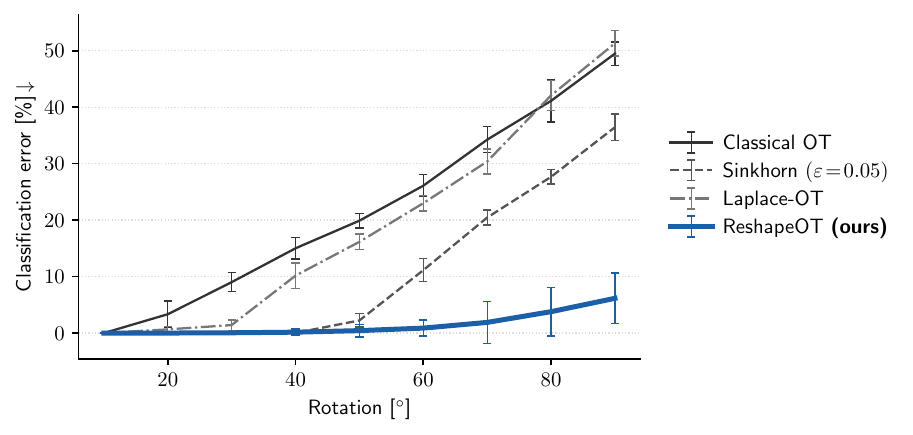}
    \caption{Domain adaptation results in terms of the SVM classification error [\%]\,$\downarrow$ for the \textit{Rotating Moons} dataset. The ablations of \mname{} use the same parametrization as their main counterparts.}
    \label{fig:moons_da_results}
\end{figure}

\medskip
In summary, \mname{} substantially reduces the domain-adaptation classification error on this benchmark by translating a small number of ground-truth displacements into a transport solution that aligns with the manifold's rotation, whereas unguided baselines are hampered by cross-class shortcuts.
This result has direct practical consequences in scientific applications: when technical batch effects place samples closer in Euclidean space than biologically meaningful variation (cf.~\cite{komenRobustFoundationModels2025}), unguided OT can match on the wrong signal. \mname{} allows a practitioner to encode prior knowledge about which directions of variation are biologically meaningful, steering the coupling away from spurious technical correlations and toward genuine biological correspondence.

\section{Conclusion}
\label{sec:conclusion}

\noindent Optimal transport is a workhorse in modeling distribution shifts, but the reliability of any OT solution is bounded by the fidelity of its ground metric. With a Euclidean metric, OT solvers exploit shortcuts that violate manifold structure both in the data and in how it is transported. \mname{} addresses this by reshaping an initial Euclidean ground metric into a Mahalanobis distance derived from second-order statistics of observed displacements, carving expressways through the cost landscape along empirically supported directions.

\smallskip

Our approach is easy to implement, minimally affects OT's compute costs, and can be extended to kernel feature spaces. We note that \mname{} is not limited to the classical OT problem formulation. Any OT algorithm that operates on a cost matrix (e.g., Sinkhorn~\cite{cuturiSinkhornDistancesLightspeed2013}, Partial-OT~\cite{chapelPartialOptimalTranport2020}, etc.) can in principle be used on top of the new ground metric.
Empirically, our approach demonstrates consistently high accuracy, outperforming a representative set of baselines across several benchmark evaluations. We further demonstrate its usefulness in two practical use cases.

\smallskip

Beyond \mname{}'s capabilities, our paper demonstrates the practical feasibility of integrating past displacements into the transport modeling task and the associated benefits. We see a particular potential for our approach in single-cell tracking (cf.~\cite{schiebinger2019optimal,kleinMappingCellsTime2025}). Building on the work of~\cite{DBLP:conf/iclr/Scarvelis023} that demonstrated the benefit of combining OT and existing data, our method, which explicitly aligns transport to past displacements, would allow specifically for the integration of longitudinal cell trajectories. 

\smallskip
Despite its strengths, \mname{} has limitations. Most importantly, it requires a representative and robust set of displacements; without them, the intrinsic transport cannot be accurately extracted, and the optimizer may be misled. Furthermore, the kernelized variant requires precise hyperparameter tuning. These parameters must be flexible enough to align with sample displacements, yet sufficiently rigid to avoid introducing spurious structural barriers or shortcuts.

\smallskip

Finally, a natural future work would be to embed \mname{} within the information-geometric framework, exploiting known links between Wasserstein geometry and statistical manifolds~\cite{Amari2018,Otto2001,Ay2024}. Concretely, sample displacements could be represented as tangent vectors on the manifold of probability distributions, enabling a principled interpolation between source and target distributions guided by prior trajectories on that manifold, thereby unifying the metric-learning and geometric perspectives pursued here.

\vspace{6pt} 

\funding{%
K.R.M.\ was supported in part by the German Federal Ministry for Research, Technology and Space (BMFTR) under grants 01IS18025A, 031L0207D, 01IS18037A, 16IS24087C.  
K.R.M.\ was also supported by the Institute of Information \& communications Technology Planning \& Evaluation (IITP) grants funded by the Korea government (MSIT) (No.\ RS-2019-II190079, Artificial Intelligence Graduate School Program of Korea University and No.\ RS-2024-00457882, AI Research Hub Project).}

\acknowledgments{
During the preparation of this manuscript, the authors used Claude Sonnet~4.6, Opus~4.6, and Gemini~3.1 for rewriting parts of the text. The authors have reviewed and edited the output and take full responsibility for the content of this publication.
}

\appendixtitles{no}
\appendixstart
\appendix

\section{Kernelization of the \mname{} Distance}
\label{app:proof-kernelization}
\noindent This appendix proves the kernel representation of the \mname{} distance presented in \cref{eq:kernel-distance}. The derivation shows that $d$ depends on the data only through inner products, and therefore admits a realization in any RKHS induced by a positive semi-definite kernel.
\begin{Proposition}[Kernelization of $d$]
Let $\lbrace \xp{x_1},\ldots,\xp{x_M}\rbrace, \lbrace \xp{y_1},\ldots,\xp{y_N}\rbrace\subset\R^d$ be source and target samples with transport coupling $\xp{\gamma_{ij}}\geq 0$, and let $\kernel\colon \R^d\times\R^d\to\R$ be a positive semi-definite kernel. Define $\Sigma\in\R^{d\times d}$ as in \cref{eq:Sigma} and $d\colon\R^d\times\R^d\to\R$ as in \cref{eq:mahalanobis}.
Write $\boldsymbol{z} = (\xp{x_1},\ldots,\xp{x_M},\xp{y_1},\ldots,\xp{y_{N}})$ for the concatenation of all samples, and let $K\in\R^{{(M+N)}\times{(M+N)}}$ be the kernel Gram matrix with $K_{kl} = \kernel(z_k, z_l)$ and $A\in\R^{(M+N)\times(M+N)}$ the block matrix
\begin{equation*}
    A = \begin{pmatrix} \mathrm{diag}(a) & -\xp{\gamma} \\ -\xp{\gamma}^\top & \mathrm{diag}(b) \end{pmatrix},
\end{equation*}
where $a_i = \sum_j\xp{\gamma_{ij}}$ and $b_j = \sum_i\xp{\gamma_{ij}}$ are the marginals.
Then $d$ admits the kernel representation
\begin{equation*}
    d(\Phi(x),\Phi(y)) = \sqrt{\kernel(x,x)-2\kernel(x,y)+\kernel(y,y)-g^\top Mg},
\end{equation*}
where $g_k = \kernel(z_k,x) - \kernel(z_k,y)$ and $M = 
\eta A^{\frac12}(I + \eta A^{\frac12}KA^{\frac12})^{-1}A^{\frac12}$.
\end{Proposition}

\begin{proof}
    Stack all training points as rows of $\boldsymbol{z} = (\xp{x_1},\ldots,\xp{x_M},\xp{y_1},\ldots,\xp{y_N})\in\R^{(M+N)\times d}$, and for each pair $(i,j)$ define the indicator vector $\alpha_{ij}\in\lbrace-1,0,+1\rbrace^{M+N}$ with $+1$ at position $i$ and $-1$ at position $M+j$. Then $\xp{x_i}-\xp{y_j} = \boldsymbol{z}^\top\alpha_{ij}$, and substituting into the definition of $\Sigma$ gives
    \begin{equation*}
        \Sigma = \sum_{ij}\xp{\gamma_{ij}}(\boldsymbol{z}^\top\alpha_{ij})(\boldsymbol{z}^\top\alpha_{ij})^\top = \boldsymbol{z}^\top\underbrace{\left( \sum\nolimits_{ij}\xp{\gamma_{ij}}\alpha_{ij}\alpha_{ij}^\top \right)}_{=A}\boldsymbol{z} = \boldsymbol{z}^\top A\boldsymbol{z}.
    \end{equation*}
    One verifies that $A$ has the stated block structure. 

    Applying the Woodbury identity to $(I + \eta \boldsymbol{z}^\top A\boldsymbol{z})^{-1}$ yields
    \begin{equation*}
        \left(I + \eta \boldsymbol{z}^\top A\boldsymbol{z}\right)^{-1} = \left(I - \boldsymbol{z}^\top \Lambda\boldsymbol{z}\right),
    \end{equation*}
    where $\Lambda = 
    \eta A^{\frac12}(I + \eta A^{\frac12}KA^{\frac12})^{-1}A^{\frac12}$
    and $K=\boldsymbol{z}\boldsymbol{z}^\top$ is the Gram matrix with $K_{kl} = z_k^\top z_l$. This converts the $d\times d$ inversion into an $(M+N)\times(M+N)$ inversion.

    Substituting into $d$ and writing $g = \boldsymbol{z}(x-y)\in\R^{M+N}$:
    \begin{align*}
        d(x,y)^2 &= (x-y)^\top\left(I-\boldsymbol{z}^\top M\boldsymbol{z}\right)(x-y) \\
        &= \|x-y\|^2-g^\top Mg.
    \end{align*}
    Every quantity depends on the data only through inner products: $K_{kl} = z_k^\top z_l$, $g_k = z_k^\top x - z_k^\top y$, and $\|x-y\|^2 = x^\top x-2x^\top y+y^\top y$. The matrix $A$ depends only on $\xp{\gamma}$. Replacing each inner product $u^\top v$ by the kernel evaluation $\kernel(u, v)$ gives the stated formula, which is valid in any RKHS induced by $\kernel$.
\end{proof}

If displacement data is provided as a set of paired observations $(\xp{x_l}, \xp{y_l})_{l=1}^N$, and marginals are modeled as uniform, then $A$ has the simple structure
\begin{equation*}
    A = \frac1N\begin{pmatrix} I & -I \\ -I & I\end{pmatrix}
\end{equation*}
where $A^{\frac12}$ can be computed in closed form as
\begin{equation*}
    A^{\frac12} = \frac{1}{\sqrt{2N}}\begin{pmatrix} I & -I \\ -I & I\end{pmatrix}
\end{equation*}

\reftitle{References}

\clearpage
\begin{center}
    \vspace*{5mm}
    {\LARGE \textsc{Supplementary Notes}}
    \vskip 5mm
\end{center}
\bigskip

\renewcommand{\appendixname}{Supplementary Note}
\renewcommand{\theproposition}{S\arabic{proposition}}
\renewcommand{\thefigure}{S\arabic{figure}}
\renewcommand{\thetable}{S\arabic{table}}
\renewcommand{\theHproposition}{Supp.\arabic{proposition}}
\renewcommand{\theHfigure}{Supp.\arabic{figure}}
\renewcommand{\theHtable}{Supp.\arabic{table}}
\setcounter{proposition}{0}
\setcounter{figure}{0}
\setcounter{table}{0}
\setcounter{section}{0}
\renewcommand{\theHsection}{Supp.\arabic{section}}
\section{Time-series Dataset Details}
\label{app:A-time-series-datasets}
\noindent The \textit{Air Quality}~\cite{air_quality_360} and \textit{Appliances}~\cite{appliances_energy_prediction_374} time-series datasets used in \mcref{sec:experiments} were preprocessed in line with~\cite{Naumann2026}, with the following main changes:
\begin{itemize}
    \item We use a \texttt{RobustScaler} (based on the median and IQR) instead of a \texttt{StandardScaler} (based on the mean and standard deviation) to improve robustness to outliers.
    \item For this reason, we also do not remove outlier instances. Outlier features, however, are removed as in~\cite{Naumann2026}.
    \item In case of \textit{Appliances}, we use the \textit{median} instead of the mean for resampling the time-series to 1h-intervals. \textit{Air Quality} already comes sampled in 1h-intervals.
\end{itemize}
The splitting into the unpaired subsets is described in~\cite{Naumann2026}.
Furthermore, we consider the same time intervals, i.e., all 123 time pairs spanning shifts of 1~to~6 hours.
Ground-truth pairing for a sample at time $t_d$ is naturally provided by the time-series record at time $t_d+h$ (with $0 \leq t_d < t_d + h \leq 23$ and $h \in \{1,2,3,4,5,6\}$) on the same day $d$ (cf.~\cite{Naumann2026}).

\noindent For computing the `transport error' metric, all methods use a barycentric mapping~\cite{DBLP:journals/siamis/FerradansPPA14,courtyOptimalTransportDomain2017} to compute the transport induced by their coupling.

\section{Synthetic Rotating Moons Generation}
\label{app:B-moons-generation}
\noindent Below, we show the Python script used to generate the synthetic \textit{Rotating Moons} data for a given rotation, as used in \mcref{sec:experiments,sec:use_case_da}.
It largely follows the experimental description from~\cite{courtyOptimalTransportDomain2017}.
\begin{lstlisting}[
language=Python,
caption=Python code to generate the synthetic \textit{Rotating Moons} dataset.,
label={lst:moons},
]
import numpy as np
from sklearn.datasets import make_moons
from sklearn.utils import check_random_state

def make_moons_da(
    rotation: float,
    n_per_moon: int = 150,
    noise: float = 0.05,
    n_test: int = 1000,
    random_state=None,
):
    rng = check_random_state(random_state)

    # --- Source: standard two moons ---------------------------------
    Xs, ys = make_moons(n_samples=2 * n_per_moon, noise=noise, random_state=rng)
    Xs[:, 0] -= 0.5

    # --- Target: rotate source points -------------------------------
    theta = np.radians(-rotation)
    R = np.array([[np.cos(theta), -np.sin(theta)], [np.sin(theta), np.cos(theta)]])
    Xt = Xs @ R
    yt = ys.copy()

    # --- Independent test set from the target distribution ----------
    Xtest, ytest = make_moons(n_samples=n_test, noise=noise, random_state=rng)
    Xtest[:, 0] -= 0.5
    Xtest = Xtest @ R

    return Xs, ys, Xt, yt, Xtest, ytest
\end{lstlisting}

\section{Hyperparameter Sensitivity of \mname{}}
\label{app:C-hyperparameter-selection}
\noindent This supplementary note provides ablations of \mname{}'s hyperparameters.
In particular, \scref{fig:rotating_time_series_hp_grids,fig:rotating_moons_hp_grids} illustrate the effects of the precisions $\lambda$ of the RBF kernel $\kernel(x,y)=\exp(-\lambda \|x-y\|^2)$ in combination with the $\Sigma$ regularization scaling parameter $\alpha$ with respect to the mean transport error on the evaluated datasets (cf.\ \mcref{sec:evaluation_metrics,sec:datasets}).

\begin{figure}[H]
    \centering
    \includegraphics[width=1.0\linewidth]{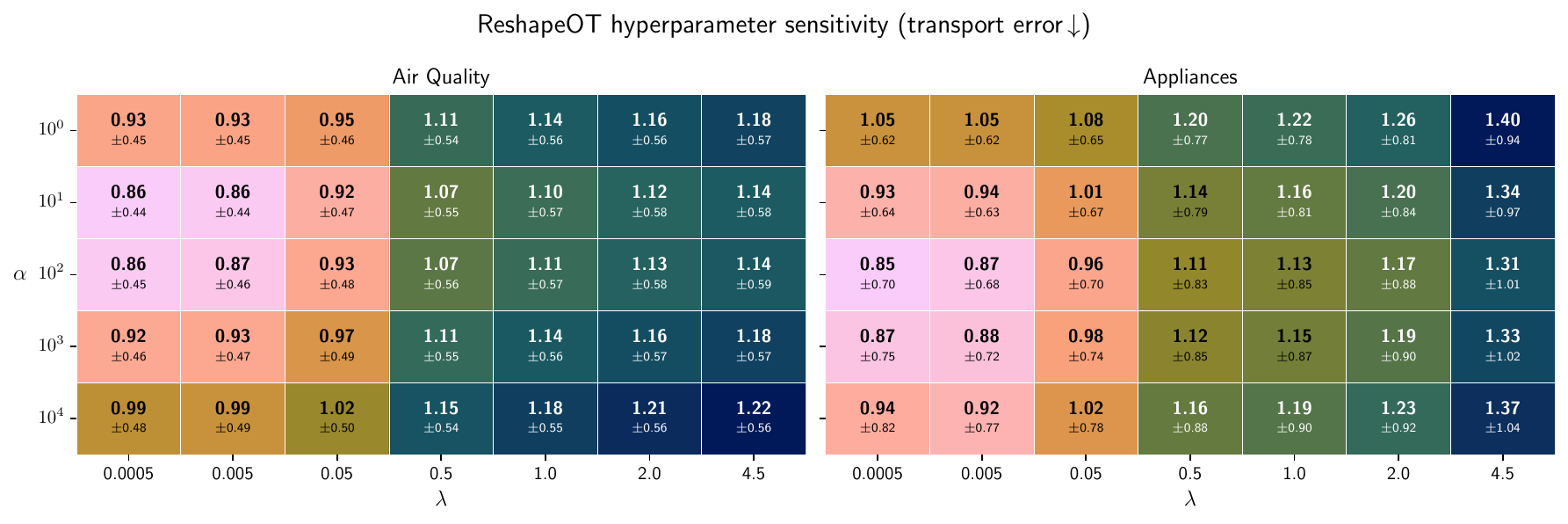}
    \caption{Averaged transport errors (cf.\ \mcref{sec:evaluation_metrics}) over all \textit{Air Quality} and \textit{Appliances} shift delays (1h--6h) for \mname{}'s $(\lambda,\alpha)$ hyperparameter combinations.}
    \label{fig:rotating_time_series_hp_grids}
\end{figure}
\begin{figure}[H]
    \centering
    \includegraphics[width=1.0\linewidth]{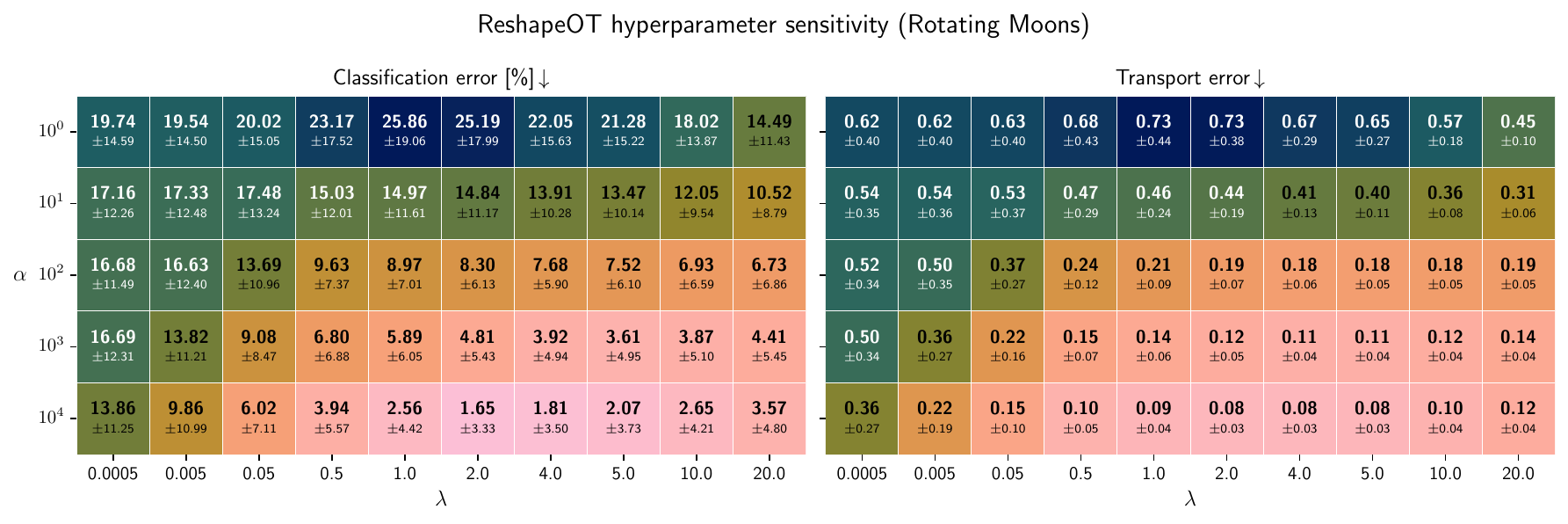}
    \caption{Averaged evaluation metrics over all \textit{Rotating Moons} rotations (10$^\circ$\,--\,90$^\circ$) for \mname{}'s $(\lambda,\alpha)$ hyperparameter combinations. `Classification error' is the domain adaptation metric of \mcref{sec:use_case_da}, and `transport error' corresponds to the metric introduced in \mcref{sec:evaluation_metrics}.}
    \label{fig:rotating_moons_hp_grids}
\end{figure}

\section{Classical OT in RBF Kernel Space}
\label{app:D-classical-ot-rbf}
\noindent In \scref{fig:rotating_moons_rbf_no_experience}, we show an ablation of \emd{} in RBF kernel space for different precisions $\lambda$ in terms of the domain adaptation classification error on the \textit{Rotating Moons} task. As we see, none of the evaluated kernel precisions improves over the baseline performance of \emd{} in the input space.

\begin{figure}[H]
    \centering
    \includegraphics[width=0.8\linewidth]{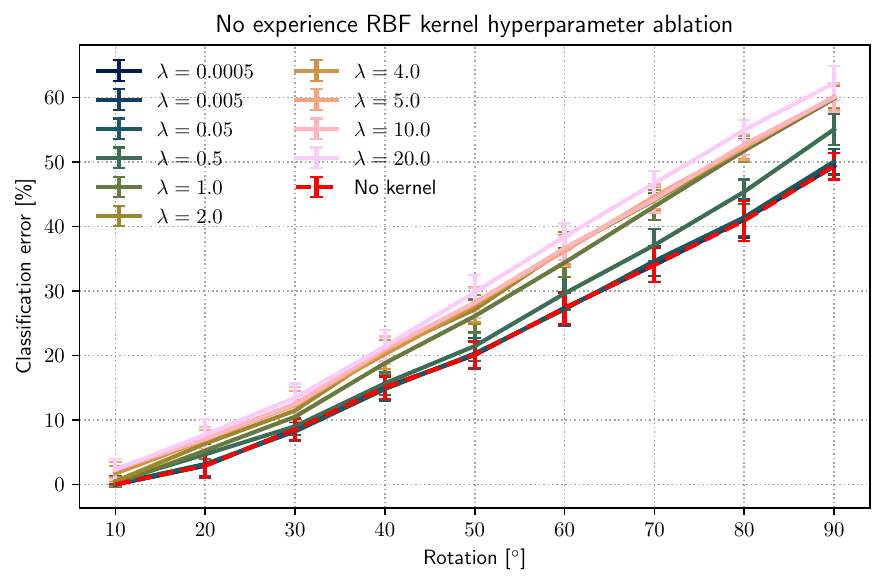}
    \caption{Ablation on applying \emd{} in the RBF kernel space (i.e.~without displacement guidance) at different $\lambda$ precisions for the \textit{Rotating Moons} domain adaptation task of \mcref{sec:use_case_da} measured by the classification error [\%]\,$\downarrow$.}
    \label{fig:rotating_moons_rbf_no_experience}
\end{figure}

\section{Bird Migration Data}
\label{app:E-bird-data}
\noindent The datasets used in \mcref{sec:use_case_migration} are publicly available from the MoveBank repository~\cite{movebank}, and we list the specific studies in \scref{tab:movebank_studies}, and selected trajectories for \mname{} guidance in \scref{tab:movebank_experience_points}.
The subsets we used in the main text span mid-September to mid-December, covering autumn migration.
Some birds have been tracked over multiple years, so they can recur when a new migration period starts.
After preprocessing, 2266 source--target pairs remain in the training data, and 47 pairs remain for the guidance of \mname{}.
Note that the \textit{Galapagos Albatrosses} study for the considered time frame contains barely any movement, as these birds do not migrate in that period. Rather, it acts as natural noise as found in real-world recorded data.
In general, some birds may also settle early after finishing their migration within the time frame, thereby becoming mostly stationary data points.
Finally, \textit{NYSDEC Raptor Tracking} includes four bird species, whereas all other datasets include only a single species.
\begin{table}[H]
    \footnotesize
    \centering
    \caption{Data studies from MoveBank~\cite{movebank} used in \mcref{sec:use_case_migration}.}
    \begin{tabular}{l r r}
        \toprule
        \textbf{Study name} & \textbf{\# Birds} & \textbf{\# Samples} \\
        \midrule
         Galapagos Albatrosses & \multirow{2}{*}{4} & \multirow{2}{*}{14} \\
         \multicolumn{3}{l}{{\scriptsize \url{https://dx.doi.org/10.5441/001/1.3hp3s250}}}\\
         \midrule
         Migration of Sabine's gulls from the Canadian High Arctic & \multirow{2}{*}{21} & \multirow{2}{*}{243} \\
         \multicolumn{3}{l}{{\scriptsize \url{https://dx.doi.org/10.5441/001/1.c745vb70}}}\\
         \midrule
         NYSDEC Raptor Tracking & \multirow{2}{*}{61} & \multirow{2}{*}{1300} \\
         \multicolumn{3}{l}{{\scriptsize \url{https://dx.doi.org/10.5441/001/1.s65q50j0}}}\\
         \midrule
         Turkey vultures in North and South America & \multirow{2}{*}{19} & \multirow{2}{*}{603} \\
         \multicolumn{3}{l}{{\scriptsize \url{https://dx.doi.org/10.5441/001/1.46ft1k05}}}\\
         \midrule
         Pandion haliaetus Osprey - SouthEast Michigan & 17 & 153 \\
         \multicolumn{3}{l}{{\scriptsize \url{https://www.movebank.org/cms/webapp?gwt_fragment=page=studies,path=study10204361}}}\\
        \bottomrule
    \end{tabular}
    \label{tab:movebank_studies}
\end{table}
\begin{table}[H]
    \footnotesize
    \centering
    \caption{Reference trajectories used to guide \mname{} in \mcref{sec:use_case_migration}.}
    \begin{tabular}{l c c}
    \toprule
    \textbf{ID} & \textbf{Year} & \textbf{\# Samples}\\
    \midrule
    Turkey vultures in North and South America\_Steamhouse 2 & 2012 & 13\\
    Migration of Sabine's gulls from the Canadian High Arctic\_BI & 2008 & 11\\
    Migration of Sabine's gulls from the Canadian High Arctic\_AD & 2011 & 13\\
    NYSDEC Raptor Tracking-argos\_BAEA 0629-30002 E63 & 2001 & 10\\
    \bottomrule
    \end{tabular}
    \label{tab:movebank_experience_points}
\end{table}

\section{Further Domain Adaptation Results}
\label{app:F-further-da-results}
\noindent We provide additional results to \mcref{fig:moons_da_results} from the domain adaptation experiment of \mcref{sec:use_case_da} in \scref{tab:rotating-moons}. In particular, we provide more rotation degrees, and Sinkhorn~\cite{cuturiSinkhornDistancesLightspeed2013} results at different regularization levels.
All methods use a barycentric mapping~\cite{DBLP:journals/siamis/FerradansPPA14,courtyOptimalTransportDomain2017} to compute the transport induced by their coupling.
\begin{table}[H]
  \centering
  \caption{Classification error [$\%$]\,$\downarrow$ on the \textit{Rotating Moons} domain adaptation task across target rotation angles. Higher angles generally correspond to more challenging shifts. Means over 20 repetitions. Values: mean $\pm$ standard deviation of error\%. \textbf{Bold} = best, \textit{italic} = second best per row.}
  \label{tab:rotating-moons}
  \resizebox{\textwidth}{!}{%
  \begin{tabular}{lccccccccc}
    \toprule
     & \multicolumn{5}{c}{Baselines} & \multicolumn{2}{c}{\textbf{Ours}} & \multicolumn{2}{c}{Ablations} \\
    \cmidrule(lr){2-6} \cmidrule(lr){7-8} \cmidrule(lr){9-10}
    Rotation & Classical & Sinkhorn & Sinkhorn & Sinkhorn & Laplace-OT & \mname{} & \mname{} & \mname{} & \mname{} \\
     & OT & {\small($\varepsilon=0.01$)} & {\small($\varepsilon=0.05$)} & {\small($\varepsilon=0.1$)} &  & {\small(Linear)} & {\small(RBF)} & {\small(Linear, $\xp{\gamma_\mathrm{rng}}$)} & {\small(RBF, $\xp{\gamma_\mathrm{rng}}$)} \\
    \midrule
    10$^\circ$ & $\mathbf{0.00_{\pm 0.00}}$ & $\mathbf{0.00_{\pm 0.00}}$ & $0.01_{\pm 0.02}$ & $0.34_{\pm 0.23}$ & $\mathbf{0.00_{\pm 0.00}}$ & $0.68_{\pm 0.96}$ & $\mathbf{0.00_{\pm 0.00}}$ & $1.87_{\pm 1.67}$ & $3.64_{\pm 3.04}$ \\
    20$^\circ$ & $3.34_{\pm 2.24}$ & $0.11_{\pm 0.21}$ & $\mathit{0.01_{\pm 0.03}}$ & $0.40_{\pm 0.50}$ & $0.65_{\pm 0.49}$ & $3.43_{\pm 1.93}$ & $\mathbf{0.00_{\pm 0.00}}$ & $9.31_{\pm 2.21}$ & $21.54_{\pm 5.15}$ \\
    30$^\circ$ & $9.01_{\pm 1.62}$ & $0.48_{\pm 0.46}$ & $\mathbf{0.00_{\pm 0.00}}$ & $0.42_{\pm 0.34}$ & $1.42_{\pm 0.91}$ & $7.57_{\pm 1.54}$ & $\mathit{0.05_{\pm 0.20}}$ & $15.84_{\pm 2.60}$ & $27.55_{\pm 5.01}$ \\
    40$^\circ$ & $15.01_{\pm 1.84}$ & $8.22_{\pm 1.33}$ & $\mathbf{0.03_{\pm 0.09}}$ & $0.94_{\pm 0.63}$ & $10.16_{\pm 2.21}$ & $10.97_{\pm 1.92}$ & $\mathit{0.16_{\pm 0.51}}$ & $23.01_{\pm 2.74}$ & $35.55_{\pm 6.80}$ \\
    50$^\circ$ & $19.90_{\pm 1.28}$ & $14.47_{\pm 0.81}$ & $\mathit{2.24_{\pm 1.15}}$ & $2.70_{\pm 1.24}$ & $16.16_{\pm 1.34}$ & $14.48_{\pm 1.86}$ & $\mathbf{0.43_{\pm 1.07}}$ & $29.03_{\pm 3.14}$ & $39.62_{\pm 7.57}$ \\
    60$^\circ$ & $26.09_{\pm 1.89}$ & $20.25_{\pm 1.02}$ & $11.13_{\pm 1.99}$ & $\mathit{4.29_{\pm 1.55}}$ & $22.99_{\pm 1.32}$ & $19.12_{\pm 1.59}$ & $\mathbf{0.88_{\pm 1.37}}$ & $34.62_{\pm 3.88}$ & $42.60_{\pm 8.34}$ \\
    70$^\circ$ & $34.25_{\pm 2.26}$ & $28.68_{\pm 2.06}$ & $20.45_{\pm 1.26}$ & $\mathit{8.02_{\pm 1.47}}$ & $30.38_{\pm 2.15}$ & $24.29_{\pm 2.06}$ & $\mathbf{1.89_{\pm 3.60}}$ & $39.39_{\pm 3.72}$ & $47.14_{\pm 6.99}$ \\
    80$^\circ$ & $41.12_{\pm 3.63}$ & $40.42_{\pm 2.54}$ & $27.66_{\pm 1.26}$ & $\mathit{15.87_{\pm 2.66}}$ & $42.08_{\pm 2.56}$ & $30.86_{\pm 2.65}$ & $\mathbf{3.78_{\pm 4.22}}$ & $44.88_{\pm 3.88}$ & $52.65_{\pm 9.21}$ \\
    90$^\circ$ & $49.48_{\pm 2.02}$ & $49.14_{\pm 2.00}$ & $36.45_{\pm 2.28}$ & $\mathit{27.37_{\pm 3.71}}$ & $51.30_{\pm 2.19}$ & $36.19_{\pm 2.05}$ & $\mathbf{6.15_{\pm 4.34}}$ & $51.13_{\pm 3.80}$ & $53.37_{\pm 9.72}$ \\
    \bottomrule
  \end{tabular}}

\end{table}

\noindent Although Sinkhorn with $\varepsilon\!=\!0.1$ performs better at higher rotation degrees, we chose Sinkhorn with $\varepsilon\!=\!0.05$ in the main text experiment as it provides a good overall trade-off between regularization and performance.

\end{document}